\documentclass{article}
\usepackage{PRIMEarxiv}
\usepackage{amssymb}
\usepackage{rotating}
\usepackage{amsmath}
\usepackage[table]{xcolor} 
\usepackage{pgfmath} 

\usepackage{bm}  
\usepackage{algorithm}
\usepackage{algpseudocode}
\usepackage{algorithm}
\usepackage{hyperref}
\usepackage[capitalize]{cleveref}
\usepackage{multirow}
\usepackage{siunitx} 
\usepackage{enumitem}

\usepackage{threeparttable}  
\usepackage{array}  
\usepackage{graphicx}  
\usepackage{threeparttable}  
\usepackage{rotating}  
\usepackage{xcolor}  
\usepackage{adjustbox}
\usepackage{booktabs}
\usepackage{xcolor}   
\usepackage{multirow}
\usepackage{graphicx}    
\usepackage{subcaption}  
\usepackage{float}       
\usepackage{rotating} 
\usepackage{tabularx} 
\usepackage[utf8]{inputenc} 
\usepackage[T1]{fontenc}    
\usepackage{hyperref}       
\usepackage{url}            
\usepackage{booktabs}       
\usepackage{amsfonts}       
\usepackage{nicefrac}       
\usepackage{microtype}      
\usepackage{lipsum}
\usepackage{fancyhdr}       
\usepackage{graphicx}       
\graphicspath{{media/}}     
\usepackage{amsmath}        
\usepackage{algorithm}
\usepackage{algpseudocode}
\usepackage{amssymb,amsfonts}  
\usepackage{amsthm}  
\usepackage{threeparttable}  
\usepackage{lipsum}
\usepackage{subcaption}
\newtheorem{theorem}{Theorem}

\title{Calibrated Prediction Set in Fault Detection with Risk Guarantees via Significance Tests
\thanks{\textit{\underline{Citation}}: 
\textbf{The four authors contribute equally to this work.}} 
}

\author{
  Mingchen Mei$^{1,\footnotemark[1]}$\\
  School of Materials Science and Engineering\\
  Beijing Institute of Technology \\
  Beijing\\
  \texttt{\{Mingchen Mei\}18839627669@163.com} \\
   \And
  Yi Li$^{1,\footnotemark[1]}$ \\
  College of Materials Science and Engineering\\
  Hunan University \\
  Changsha, Hunan\\
   \texttt{\{Yi Li\}zhonghuajia@hnu.edu.cn} \\
   \And
  YiYao Qian$^{1,\footnotemark[1]}$ \\
  School of Electronic Science and Engineering\\
  University of Electronic Science and Technology of China \\
  Chengdu, Sichuan\\
   \texttt{\{YiYao Qian\}18836093633@163.com} \\
   \And
  Zijun Jia$^{1,\footnotemark[1]}$ \\
  School of Automation Science and Electrical Engineering\\
  Beihang University\\
  Beijing\\
   \texttt{\{Zijun Jia\}3045973453@qq.com} \\
  }

\begin{document}
\maketitle

\begin{abstract}
Fault detection is crucial for ensuring the safety and reliability of modern industrial systems. However, a significant scientific challenge is the lack of rigorous risk control and reliable uncertainty quantification in existing diagnostic models, particularly when facing complex scenarios such as distributional shifts. To address this issue, this paper proposes a novel fault detection method that integrates significance testing with the conformal prediction framework to provide formal risk guarantees. The method transforms fault detection into a hypothesis testing task by defining a nonconformity measure based on model residuals. It then leverages a calibration dataset to compute p-values for new samples, which are used to construct prediction sets mathematically guaranteed to contain the true label with a user-specified probability, $1-\alpha$. Fault classification is subsequently performed by analyzing the intersection of the constructed prediction set with predefined normal and fault label sets. Experimental results on cross-domain fault diagnosis tasks validate the theoretical properties of our approach. The proposed method consistently achieves an empirical coverage rate at or above the nominal level ($1-\alpha$), demonstrating robustness even when the underlying point-prediction models perform poorly. Furthermore, the results reveal a controllable trade-off between the user-defined risk level ($\alpha$) and efficiency, where higher risk tolerance leads to smaller average prediction set sizes. This research contributes a theoretically grounded framework for fault detection that enables explicit risk control, enhancing the trustworthiness of diagnostic systems in safety-critical applications and advancing the field from simple point predictions to informative, uncertainty-aware outputs.
\end{abstract}
\section{Introduction}

In industrial systems and complex engineering, fault detection is a core component for maintaining system reliability and operational safety~\cite{khandelwal2025fault}. Bearing fault diagnosis is particularly critical, as a failure can trigger a cascade of catastrophic consequences, including production shutdowns, equipment damage, and even personnel casualties~\cite{pinedo2020vibration}. However, traditional fault detection methods heavily rely on prior assumptions about data distributions~\cite{zhu2024review}. When faced with dynamic data distributions or cross-domain (cross-dataset) applications, their performance degrades sharply due to overfitting, making it difficult to provide reliable risk guarantees~\cite{wang-etal-2025-sconu,li2024learn,liao2021dynamic}. Therefore, constructing a fault detection framework that is free from distributional assumptions and possesses rigorous theoretical guarantees has become an urgent problem in this field.

Conformal Prediction (CP), an emerging distribution-free learning framework, offers a novel approach to address the aforementioned challenges~\cite{wang-etal-2024-conu,wang2025sample,jia2025coverage}. Its core lies in quantifying the uncertainty of unknown samples by constructing prediction sets~\cite{norinder2014introducing}. Under the fundamental assumption of data exchangeability, it guarantees that the coverage probability of these prediction sets is no lower than a pre-specified confidence level. CP assesses the abnormality of samples through a Nonconformity Measure without being constrained by the specific form of the data distribution~\cite{kato2023review}. This enables it to provide a rigorous theoretical foundation for risk control in fault detection within complex industrial scenarios characterized by diverse data sources and non-standard distributions~\cite{farouq2022conformal}.

This paper proposes a fault detection method based on significance test-calibrated prediction sets, which deeply integrates the conformal prediction framework. The method first defines a nonconformity measure based on model residuals, skillfully transforming fault detection into a hypothesis testing task: specifically, testing the null hypothesis that a new sample belongs to the normal label set. Subsequently, it calculates nonconformity scores using a calibration dataset and constructs prediction sets that satisfy coverage guarantees through p-values. Fault classification is ultimately achieved based on the intersection between the prediction sets and the normal/fault label sets. The prediction sets generated by this method ensure that the correct result is contained within the set with a probability of at least $1-\alpha$, and it strictly controls the false alarm rate (Type I error) to not exceed $\alpha$~\cite{wang2025coin}. Furthermore, the size of the prediction set is inversely related to the risk level, making it an effective and intuitive indicator for evaluating model uncertainty.

The main contributions of this study can be summarized in three points: (1) We construct a new framework for fault detection based on conformal prediction, which transforms the fault recognition problem into a significance testing task, providing diagnostic systems with a rigorously statistically significant and calibratable risk guarantee. (2) We propose a new evaluation metric, Average Prediction Set Size (APSS), as an effective tool for quantifying uncertainty in fault diagnosis tasks, intuitively reflecting the model's confidence in its predictions. (3) We validate the effectiveness and robustness of the framework through extensive cross-dataset experiments. The results strongly demonstrate the superior performance of the proposed method in handling data distribution shifts, showcasing its strong generalization ability and practical value in real-world industrial scenarios.

\section{Related Work}
\subsection{Deep Learning for Bearing Fault Diagnosis}
Bearing fault diagnosis, a critical technology for rotating machinery, has been profoundly reshaped by deep learning (DL) innovations~\cite{sun2024caterpillar}. The field's progression began with early explorations like He's~\cite{he2017deep} LAMSTAR network, which confirmed DL's potential but also revealed its limitations in noisy environments. This led to comprehensive comparative analyses, as reviewed by Zhang~\cite{zhang2020deep}, who systematically evaluated the unique strengths of various architectures like CNNs for spatial features, AEs for unsupervised learning, RNNs for temporal dynamics, and GANs for data imbalance. Subsequent research has produced more specialized solutions, such as Cui's~\cite{cui2022feature} interpretable feature-selection framework for non-stationary conditions, Chen's~\cite{chen2020deep} fusion of signal processing with CNNs to enhance performance on imbalanced data, and Xu's~\cite{xu2022bearing} hybrid approach using GANs and dynamic models to overcome data scarcity for new conditions. In essence, research has evolved from single-model applications towards multifaceted strategies involving comparative architectural analysis, scenario-specific optimization, and the integration of physical models with data-driven methods, although further breakthroughs in robustness, adaptability, and engineering practicality are still needed.

\subsection{Conformal Prediction}

To address the prevalent issue of overconfidence and lack of reliable error control in modern machine learning models, Conformal Prediction~\cite{angelopoulos2023conformal,tibshirani2023conformal,shafer2008tutorial,fontana2023conformal,barber2023conformal} offers a rigorous, distribution-free solution. As a post-hoc framework, CP can be applied to any pre-trained model to augment its predictions with statistically valid uncertainty bounds. The methodology, particularly its practical variant known as split-conformal prediction, partitions data to train a model and then calibrate its “strangeness” or non-conformity on unseen examples. This calibration process allows the construction of prediction sets that are guaranteed to achieve a desired marginal coverage rate over the long run. For instance, if a 95\% coverage level is set, no more than 5\% of the prediction sets will fail to contain the true outcome. The principal appeal of CP lies in this finite-sample guarantee, which is independent of the model architecture or data distribution. This property, combined with the adaptive nature of the prediction sets that transparently reflect model confidence, establishes CP as a foundational technique for robust decision-making in safety-conscious applications.
\section{Methodology}
\label{sec:method}

\subsection{Conformal Prediction Framework}
\label{subsec:conformal}

Conformal prediction is a distribution-free framework that constructs prediction sets with rigorous coverage guarantees. Given a dataset \( \mathcal{D} = \{(x_i, y_i)\}_{i=1}^N \) where \( x_i \in \mathcal{X} \) are feature vectors and \( y_i \in \mathcal{Y} \) are labels, the goal is to construct a prediction set \( \mathcal{C}(x_{N+1}) \) for a new instance \( x_{N+1} \) such that:

\begin{equation}
\mathbb{P}(y_{N+1} \in \mathcal{C}(x_{N+1})) \geq 1-\alpha,
\label{eq:coverage}
\end{equation}

where \( \alpha \in (0,1) \) is the significance level, and the probability is taken over the joint distribution of the training and test data.

\subsubsection{Nonconformity Measures}
A key component of conformal prediction is the nonconformity measure \( S(x, y) \), which quantifies how unusual the label \( y \) is for the instance \( x \). For fault detection, we define the nonconformity measure based on the residual error of a predictive model:

\begin{equation}
S(x, y) = |1 - \hat{f}_y(x)|,
\label{eq:nonconformity}
\end{equation}

where \( \hat{f}(x) \) is a prediction from a base model trained on the calibration set. Larger values of \( S(x, y) \) indicate greater nonconformity.

\subsubsection{Conformal Prediction Set Construction}
Given a nonconformity measure \( S \), the conformal prediction set for a new instance \( x_{N+1} \) is constructed as follows:

\begin{enumerate}
    \item Compute nonconformity scores \( s_i = S(x_i, y_i) \) for each calibration instance \( (x_i, y_i) \), \( i = 1, \ldots, N \).
    \item For each candidate label \( y \in \mathcal{Y} \), compute the nonconformity score \( S(x_{N+1}, y) \).
    \item Calculate the p-value for each candidate label \( y \):
    \begin{equation}
    p(y) = \frac{1}{N+1} \left( \sum_{i=1}^N \mathbb{I}\{s_i > S(x_{N+1}, y)\} + 1 \right),
    \label{eq:pvalue}
    \end{equation}
    where \( \mathbb{I}\{\cdot\} \) is the indicator function.
    \item The prediction set at significance level \( \alpha \) is:
    \begin{equation}
    \mathcal{C}_\alpha(x_{N+1}) = \{ y \in \mathcal{Y} : p(y) > \alpha \}.
    \label{eq:prediction_set}
    \end{equation}
\end{enumerate}

This construction ensures that the prediction set \( \mathcal{C}_\alpha(x_{N+1}) \) has coverage at least \( 1-\alpha \) under exchangeability assumptions.

\subsection{Significance Tests for Fault Detection}
\label{subsec:significance_tests}

In the context of fault detection, we frame the problem as a hypothesis testing task. For each new instance \( x_{N+1} \), we test the null hypothesis \( H_0: y_{N+1} \in \mathcal{Y}_{\text{normal}} \) against the alternative \( H_1: y_{N+1} \in \mathcal{Y}_{\text{fault}} \), where \( \mathcal{Y}_{\text{normal}} \) and \( \mathcal{Y}_{\text{fault}} \) are the sets of normal and faulty labels, respectively. The complete algorithm for calibrated fault detection is summarized in Algorithm \ref{alg:fault_detection}.

\subsubsection{p-value-based Fault Detection}
We adapt the conformal prediction framework to fault detection by:

\begin{enumerate}
    \item Constructing a prediction set \( \mathcal{C}_\alpha(x_{N+1}) \) using the method described in Section \ref{subsec:conformal}.
    \item Classifying the instance \( x_{N+1} \) as:
    \begin{itemize}
        \item \textbf{Normal} if \( \mathcal{C}_\alpha(x_{N+1}) \cap \mathcal{Y}_{\text{normal}} \neq \emptyset \) and \( \mathcal{C}_\alpha(x_{N+1}) \subseteq \mathcal{Y}_{\text{normal}} \).
        \item \textbf{Faulty} if \( \mathcal{C}_\alpha(x_{N+1}) \cap \mathcal{Y}_{\text{fault}} \neq \emptyset \).
        \item \textbf{Ambiguous} otherwise.
    \end{itemize}
\end{enumerate}

This approach ensures that the false alarm rate (i.e., incorrectly classifying a normal instance as faulty) is controlled at level \( \alpha \).

\subsubsection{Calibrated Risk Guarantees}
By construction, the proposed method provides the following risk guarantees:

\begin{theorem}[Calibrated Risk Control]
For any significance level \( \alpha \in (0,1) \), the probability of falsely detecting a fault (Type I error) is bounded by \( \alpha \):
\begin{equation}
\mathbb{P}(\text{Fault detected} \mid y_{N+1} \in \mathcal{Y}_{\text{normal}}) \leq \alpha.
\label{eq:type1_error}
\end{equation}
\end{theorem}

\begin{proof}
This follows directly from the coverage property of conformal prediction sets (Equation \ref{eq:coverage}). If \( y_{N+1} \in \mathcal{Y}_{\text{normal}} \), the probability that \( y_{N+1} \notin \mathcal{C}_\alpha(x_{N+1}) \) (leading to a false alarm) is at most \( \alpha \).
\end{proof}

\begin{algorithm}[t]
\caption{Calibrated Fault Detection with p-value Conformal Prediction}
\label{alg:fault_detection}
\begin{algorithmic}[1]
\Require Calibration dataset \( \mathcal{D} = \{(x_i, y_i)\}_{i=1}^N \), base model \( \hat{f} \), significance level \( \alpha \).
\Ensure Fault detection decision for new instance \( x_{N+1} \).

\State Compute nonconformity scores \( s_i = |1 - \hat{f}_{y_i}(x_i)| \) for \( i = 1, \ldots, N \).
\State Significance level \( \alpha\).
\State For each candidate label \( y \in \mathcal{Y} \):
\State \quad Compute nonconformity score \( S(x_{N+1}, y) = |1 - \hat{f}_y(x_{N+1})| \).
\State \quad Compute p-value \( p(y) = \frac{1}{N+1} \left( \sum_{i=1}^N \mathbb{I}\{s_i > S(x_{N+1}, y)\} + 1 \right) \).
\State Construct prediction set \( \mathcal{C}_{\alpha}(x_{N+1}) = \{ y \in \mathcal{Y} : p(y) > \alpha \} \).
\State If \( \mathcal{C}_{\alpha}(x_{N+1}) \cap \mathcal{Y}_{\text{normal}} \neq \emptyset \) and \( \mathcal{C}_{\alpha}(x_{N+1}) \subseteq \mathcal{Y}_{\text{normal}} \):
\State \quad Return "Normal".
\State Else if \( \mathcal{C}_{\alpha}(x_{N+1}) \cap \mathcal{Y}_{\text{fault}} \neq \emptyset \):
\State \quad Return "Faulty".
\State Else:
\State \quad Return "Ambiguous".
\end{algorithmic}
\end{algorithm}

\subsection{Theoretical Properties}
\label{subsec:theory}

Our method inherits several desirable theoretical properties from conformal prediction:

\begin{theorem}[Validity Under Exchangeability]
Under the assumption that the calibration data and test data are exchangeable, the prediction set \( \mathcal{C}_\alpha(x_{N+1}) \) defined in Equation \ref{eq:prediction_set} satisfies:
\begin{equation}
\mathbb{P}(y_{N+1} \in \mathcal{C}_\alpha(x_{N+1})) \geq 1-\alpha.
\label{eq:validity}
\end{equation}
\end{theorem}

\begin{theorem}[Consistency]
As the calibration set size \( N \to \infty \), the prediction set \( \mathcal{C}_\alpha(x_{N+1}) \) converges to the smallest possible set that satisfies the coverage guarantee in Equation \ref{eq:coverage}.
\end{theorem}

These properties ensure that our fault detection method provides reliable risk guarantees even with limited calibration data.
\section{Experimental Settings}
\subsection{Datasets}

\subsubsection{CWRU}
The Case Western Reserve University(CWRU)~\cite{smith2015rolling} bearing fault dataset is a benchmark dataset for machinery fault diagnosis. It comprises vibration signals collected from test rigs featuring drive-end (DE) and fan-end (FE) bearings operating under various motor loads (0 to 3 horsepower) and rotational speeds (approximately 1720 to 1797 RPM). To facilitate focused model development and ensure comparability with prior studies, this work utilizes data exclusively from the drive-end bearing under the 1 horsepower load condition. Four representative health states are selected: "Normal" (healthy operation), "Ball" fault (inner race defect), "IR" fault (inner race defect), and "OR" fault (outer race defect). A standardized sampling approach is applied to extract 1,024-point segments from the raw vibration signals at a sampling frequency of 12 kHz, generating balanced training samples for each fault category. This data curation enables robust development and evaluation of algorithms for cross-condition bearing fault classification tasks.
\subsubsection{SEU}
The Southeast University (SEU) Gearbox Dataset (~\cite{shao2018highly}) contains mechanical vibration signals collected from three core industrial components: induction motors, bearings, and gearboxes. Developed by the School of Instrument Science and Engineering at SEU, the dataset includes 6,000 induction motor time series samples under six operating conditions (e.g., healthy, rotor bar fracture); 5,000 bearing samples covering ten health states (e.g., ball bearing failure, inner/outer raceway defects); and 9,000 gearbox samples covering five fault types (e.g., gear collapse, bearing crack). Four representative bearing health states are used here, the same as in the CWRU dataset.

\subsection{Base Models}
For the purpose of a controlled comparison, we selected five foundational models: ResNet, MobileNetV3, ShuffleNetV2, SqueezeNet, and GhostNet. A key aspect of our methodology was to standardize their complexity, ensuring each had a nearly identical parameter count. Despite this uniformity in size, these architectures feature fundamentally different structural strategies. ResNet’s design directly tackles gradient degradation in deep networks through its signature shortcut connections~\cite{he2016deep}. The other models prioritize computational economy: MobileNetV3 and ShuffleNetV2 are frameworks optimized for efficiency on low-power devices~\cite{howard2019searching,ma2018shufflenet}, while SqueezeNet and GhostNet achieve extreme model compression and minimal computational load, making them ideal choices for edge computing and mobile applications~\cite{iandola2016squeezenet,han2020ghostnet}.

\subsection{Methods for Feature Extraction}
To effectively leverage the powerful feature extraction capabilities of two-dimensional Convolutional Neural Networks (2D-CNNs), the original one-dimensional (1D) time-series vibration signals were transformed into two-dimensional (2D) time-frequency representations. The Continuous Wavelet Transform (CWT) was selected for this task due to its proficiency in analyzing non-stationary signals, which are characteristic of machine fault data, and its ability to provide excellent resolution in both the time and frequency domains.

\subsection{Evaluation Metrics}
Our evaluation protocol relies on two complementary metrics. The integrity of the error control mechanism is verified using the Empirical Coverage Rate (ECR), which measures the alignment between the observed error rate and the predefined significance level. Concurrently, the Average Prediction Set Size serves a dual purpose: it acts as a direct proxy for the model’s confidence in its decisions while also gauging the efficiency of the resulting predictions.

\subsection{Hyper - parameters}
To facilitate a fair assessment of predictive uncertainty across different architectures, we first standardized the model complexity. Five distinct models—ResNet, MobileNetV3, ShuffleNetV2, SqueezeNet, and GhostNet—were specifically re-architected to achieve parameter parity, each containing approximately 1.17 million parameters (ranging from 1,167,971 to 1,177,659). For the experimental setup, both the CWRU and SEU datasets were partitioned into a 60\% training subset and a 40\% testing subset. All models were trained for 20 epochs using a learning rate of 0.001. Following the training phase, the testing subset was further divided equally, dedicating one half for calibration and the other for final evaluation. This 1:1 split between calibration and test data was consistently applied in all experiments. 
\section{Experiments}

\begin{table}[!t]
\centering
\small 
\setlength{\textfloatsep}{12pt}
\setlength{\abovecaptionskip}{8pt}
\setlength{\belowcaptionskip}{10pt}
\renewcommand{\arraystretch}{1.3}

\caption{Accuracy Comparison of Different Models}
\label{tab:cross_ser_results}

\begin{threeparttable}
\setlength{\tabcolsep}{12pt} 
\begin{tabular}{@{} l S[table-format=2.1] S[table-format=2.1] @{}}
\toprule[1.5pt]
\textbf{Model Architecture} 
& {C→S (\%)} & {S→C (\%)} \\ 
\midrule[1pt]
ResNet         & 31.25        & 30.67       \\
MobileNetV3    & 24.52        & 29.23       \\
ShuffleNetV2   & 24.52 & 31.16 \\
SqueezeNet     & 24.12        & 29.50       \\
GhostNet       & 32.29       & 36.30       \\
\midrule[1.5pt]
\multicolumn{3}{l}{\textit{\textbf{Note}}: C = CWRU; S = SEU (Cross-domain accuracy test)} \\
\multicolumn{3}{l}{\textit{Direction: Source → Target indicates training on source, testing on target.}}
\end{tabular}
\end{threeparttable}
\end{table}
\begin{figure}[!t]
    \centering
    \begin{subfigure}[b]{0.18\textwidth}
        \centering
        \includegraphics[width=\textwidth]{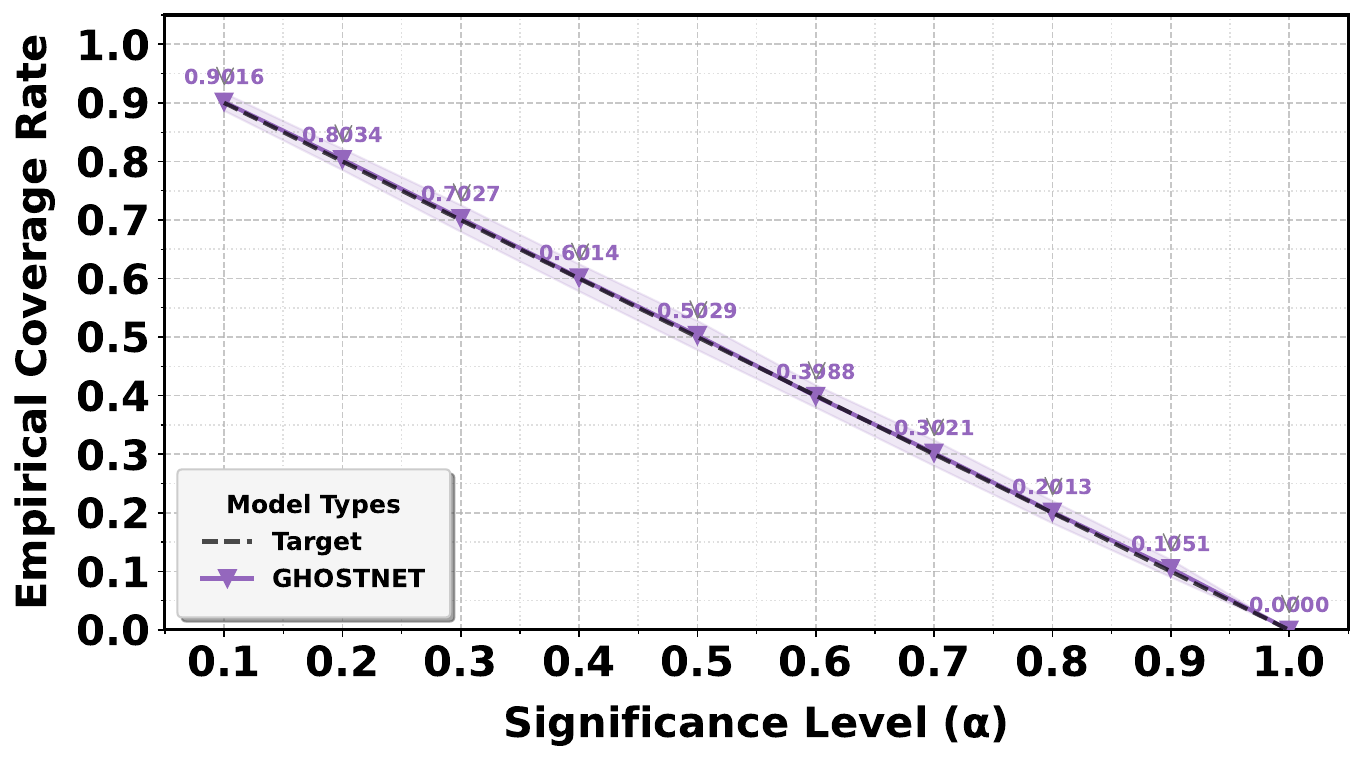}
        \caption{C→C(ghostnet)}
        \label{fig:sub1}
    \end{subfigure}
    \hfill
    \begin{subfigure}[b]{0.18\textwidth}
        \centering
        \includegraphics[width=\textwidth]{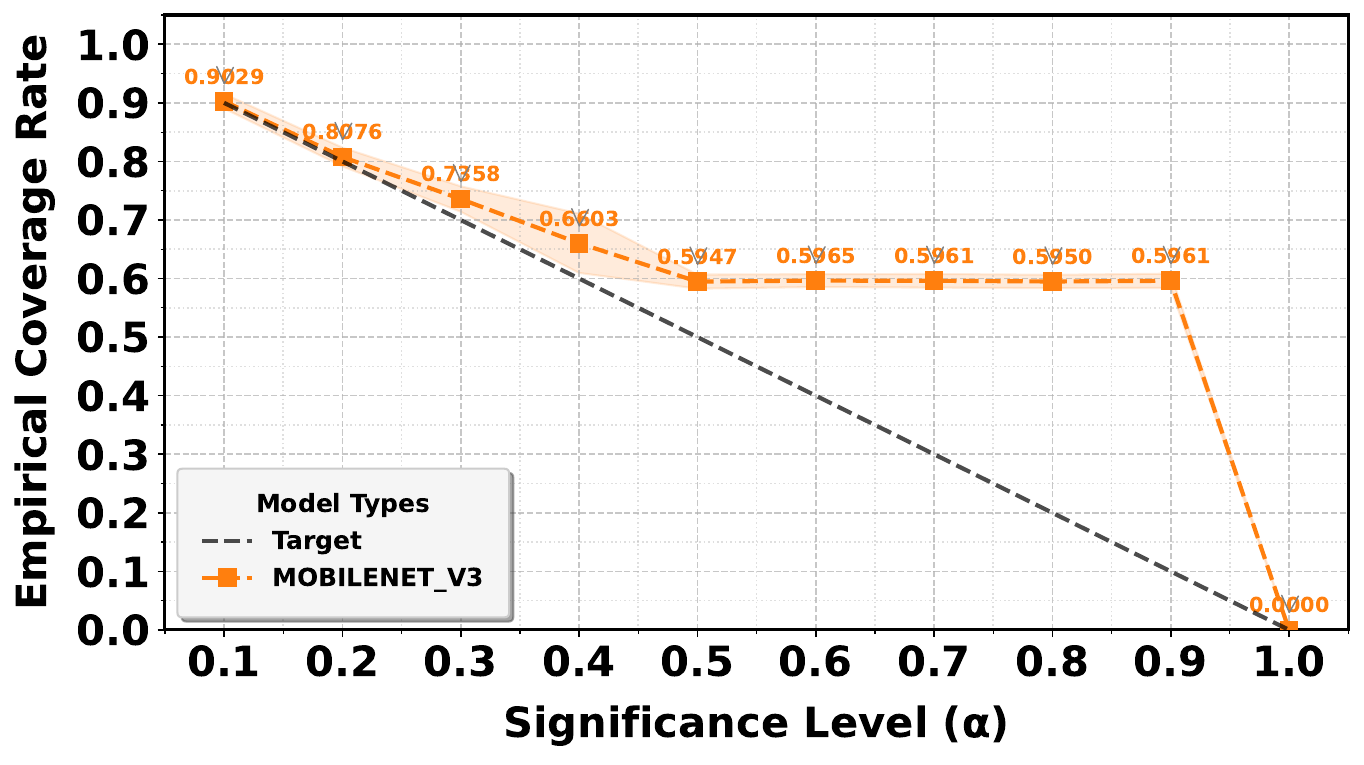}
        \caption{C→C(mobilenetv3)}
        \label{fig:sub2}
    \end{subfigure}
    \hfill
    \begin{subfigure}[b]{0.18\textwidth}
        \centering
        \includegraphics[width=\textwidth]{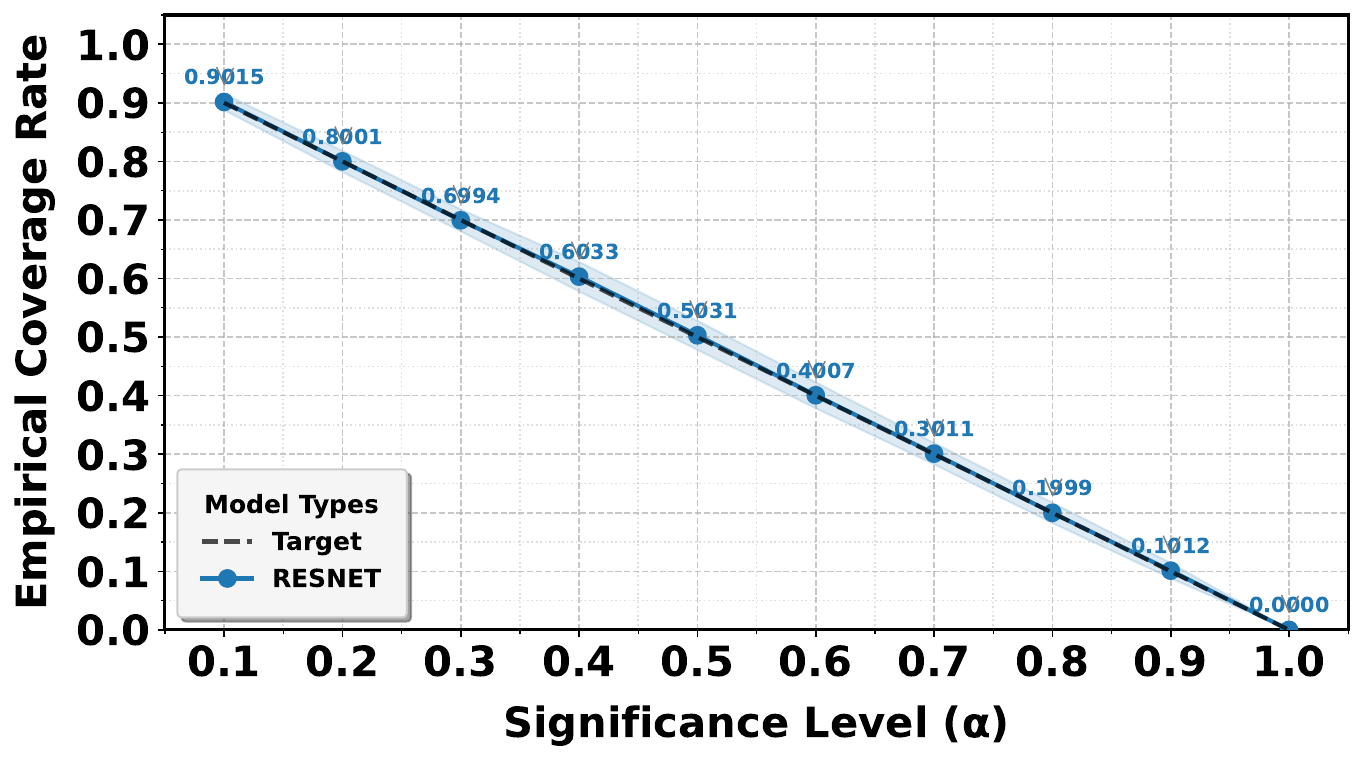}
        \caption{C→C(resnet)}
        \label{fig:sub3}
    \end{subfigure}
    \hfill
    \begin{subfigure}[b]{0.18\textwidth}
        \centering
        \includegraphics[width=\textwidth]{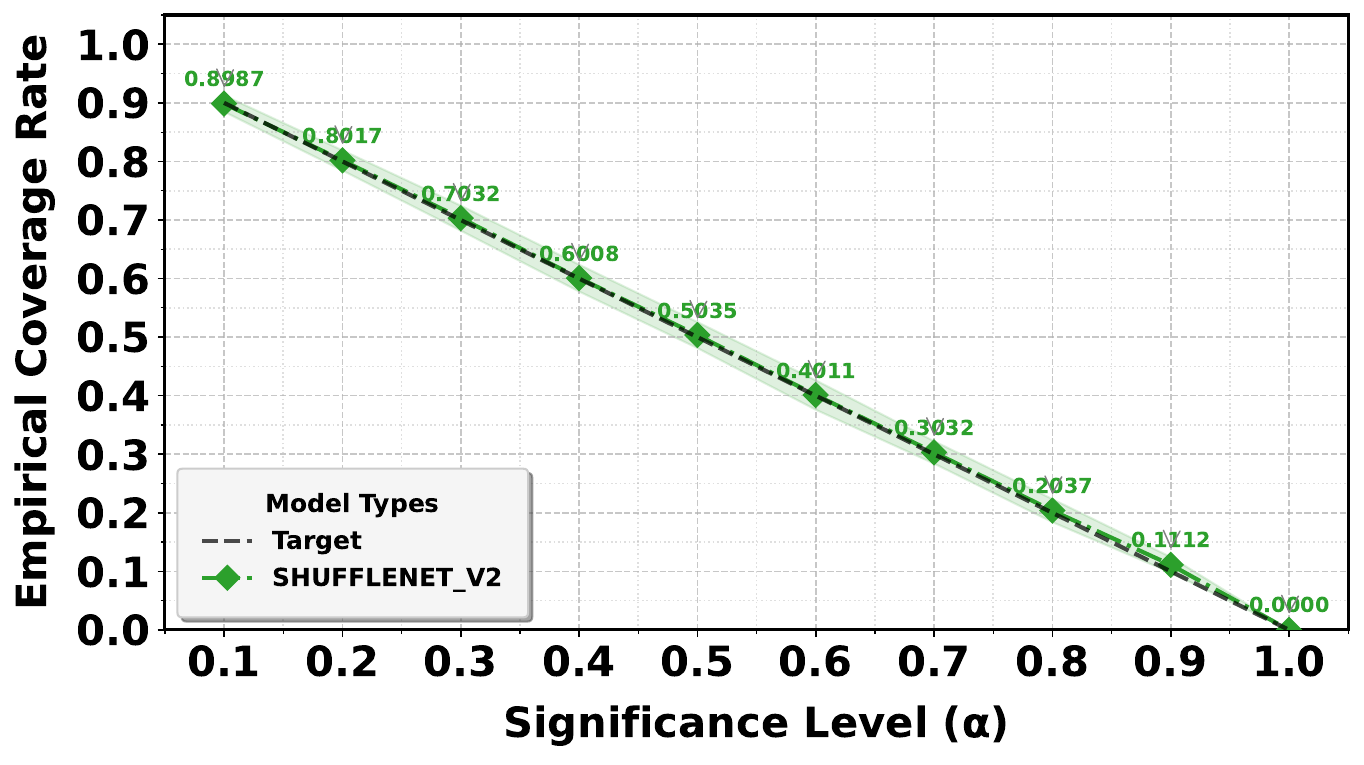}
        \caption{C→C(shufflenetv2)}
        \label{fig:sub4}
    \end{subfigure}
    \hfill
    \begin{subfigure}[b]{0.18\textwidth}
        \centering
        \includegraphics[width=\textwidth]{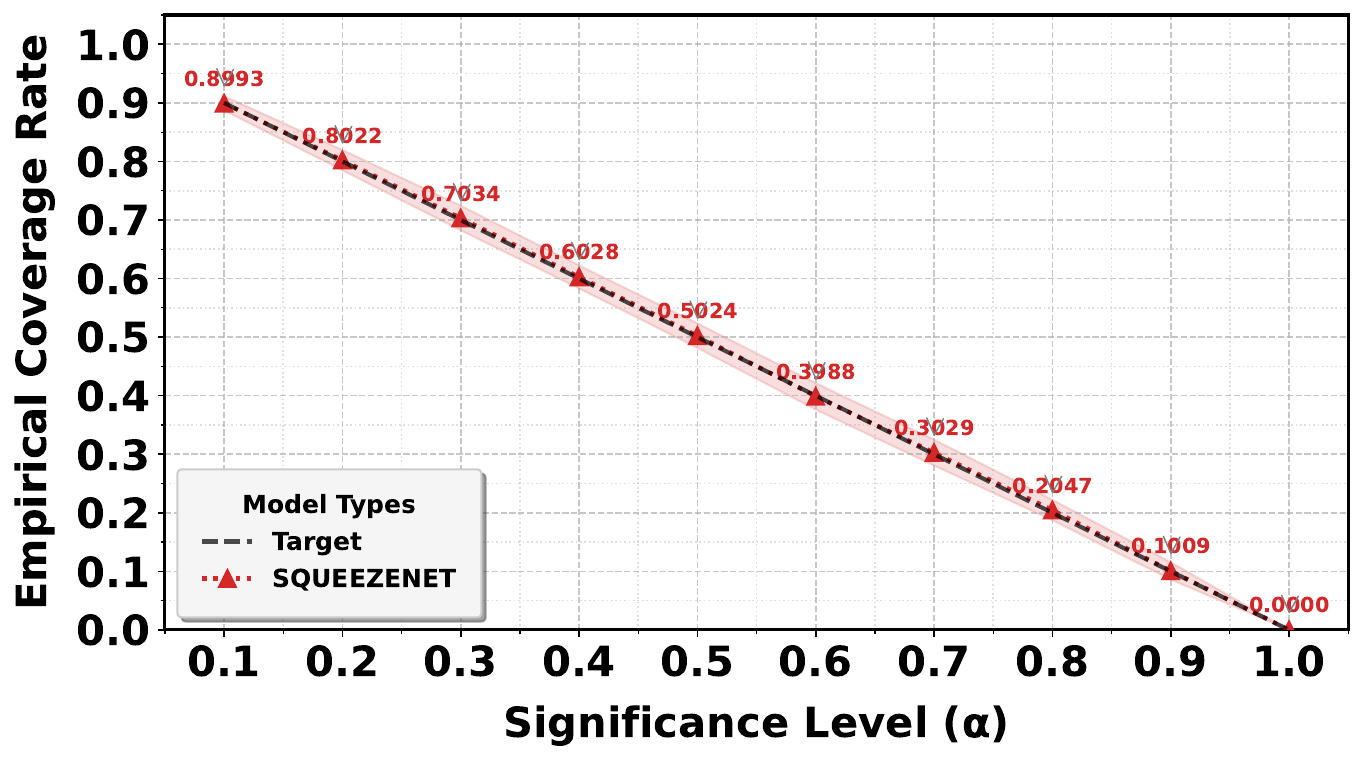}
        \caption{C→C(squeezenet)}
        \label{fig:sub5}
    \end{subfigure}
    
    \vspace{1cm} 
    
    \begin{subfigure}[b]{0.18\textwidth}
        \centering
        \includegraphics[width=\textwidth]{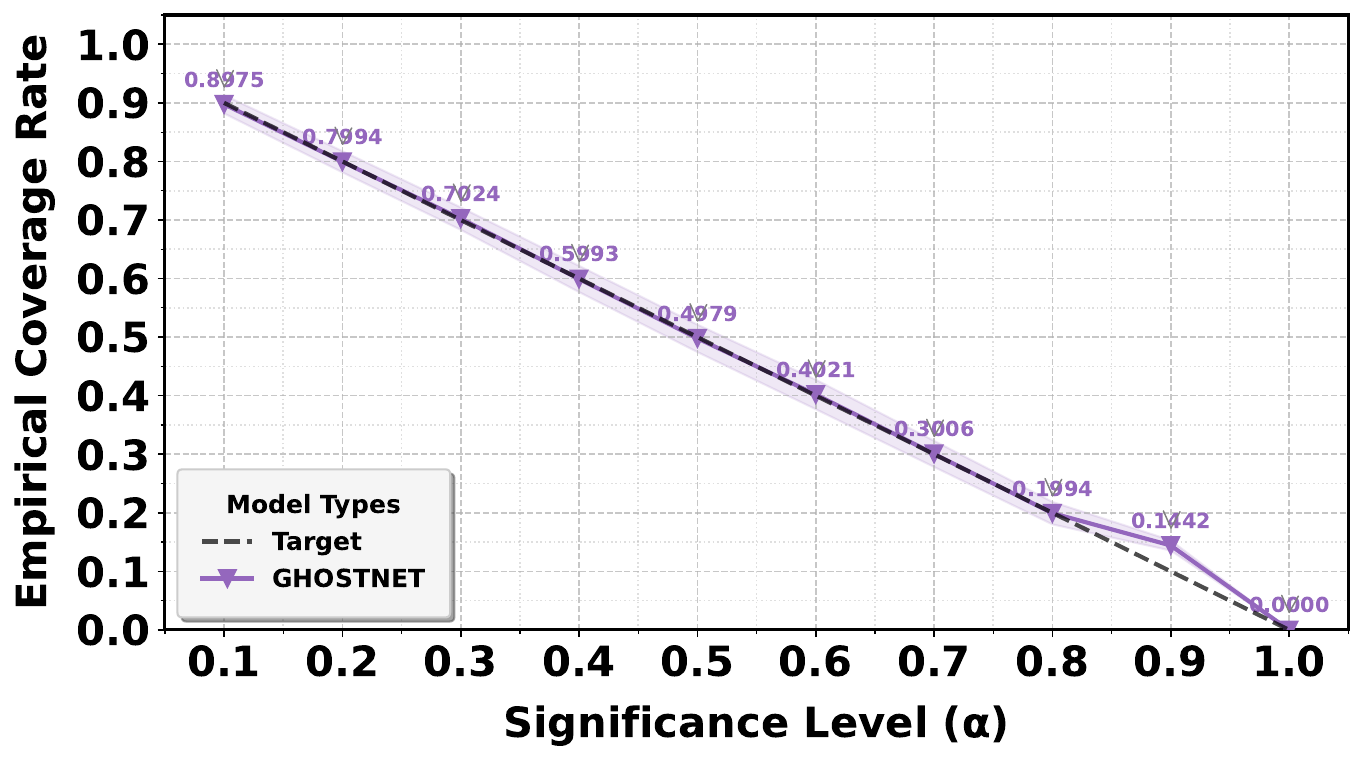}
        \caption{S→C(ghostnet)}
        \label{fig:sub6}
    \end{subfigure}
    \hfill
    \begin{subfigure}[b]{0.18\textwidth}
        \centering
        \includegraphics[width=\textwidth]{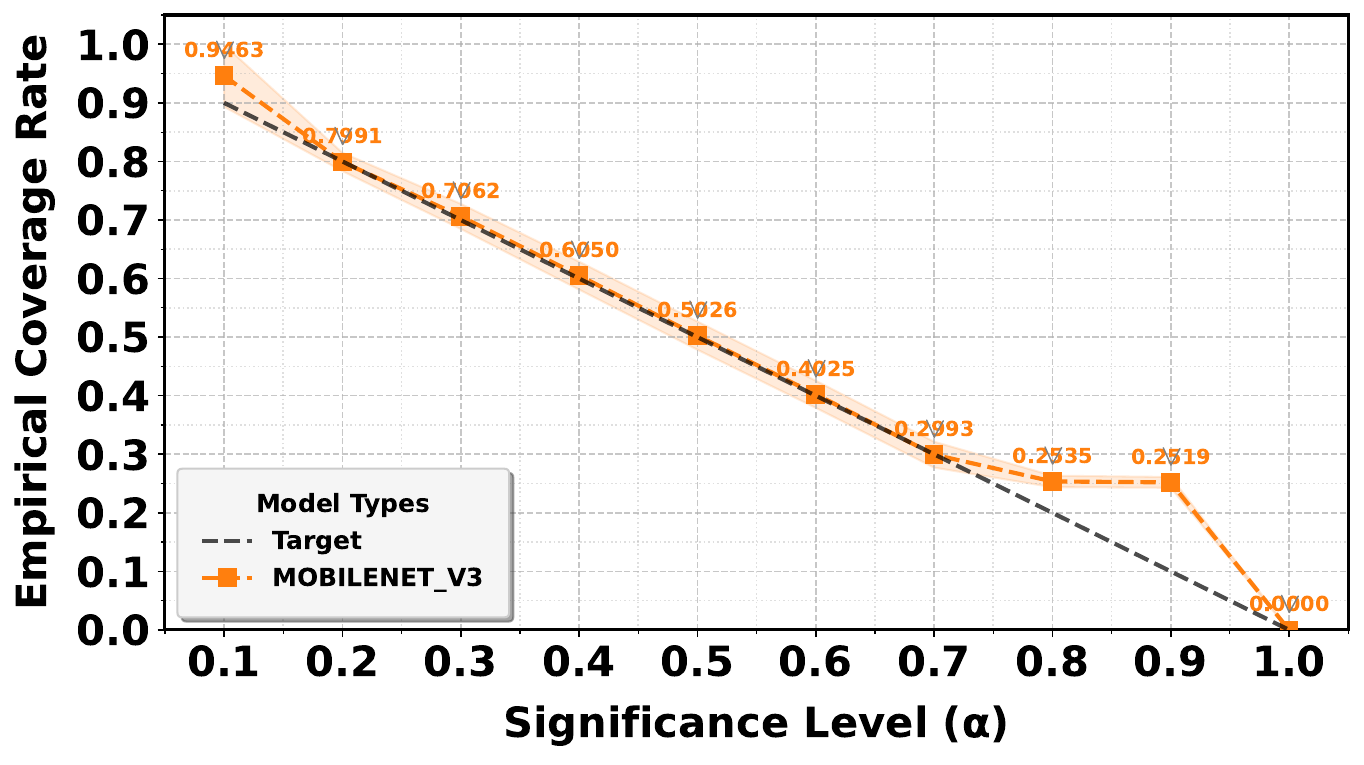}
        \caption{S→C(mobilenetv3)}
        \label{fig:sub7}
    \end{subfigure}
    \hfill
    \begin{subfigure}[b]{0.18\textwidth}
        \centering
        \includegraphics[width=\textwidth]{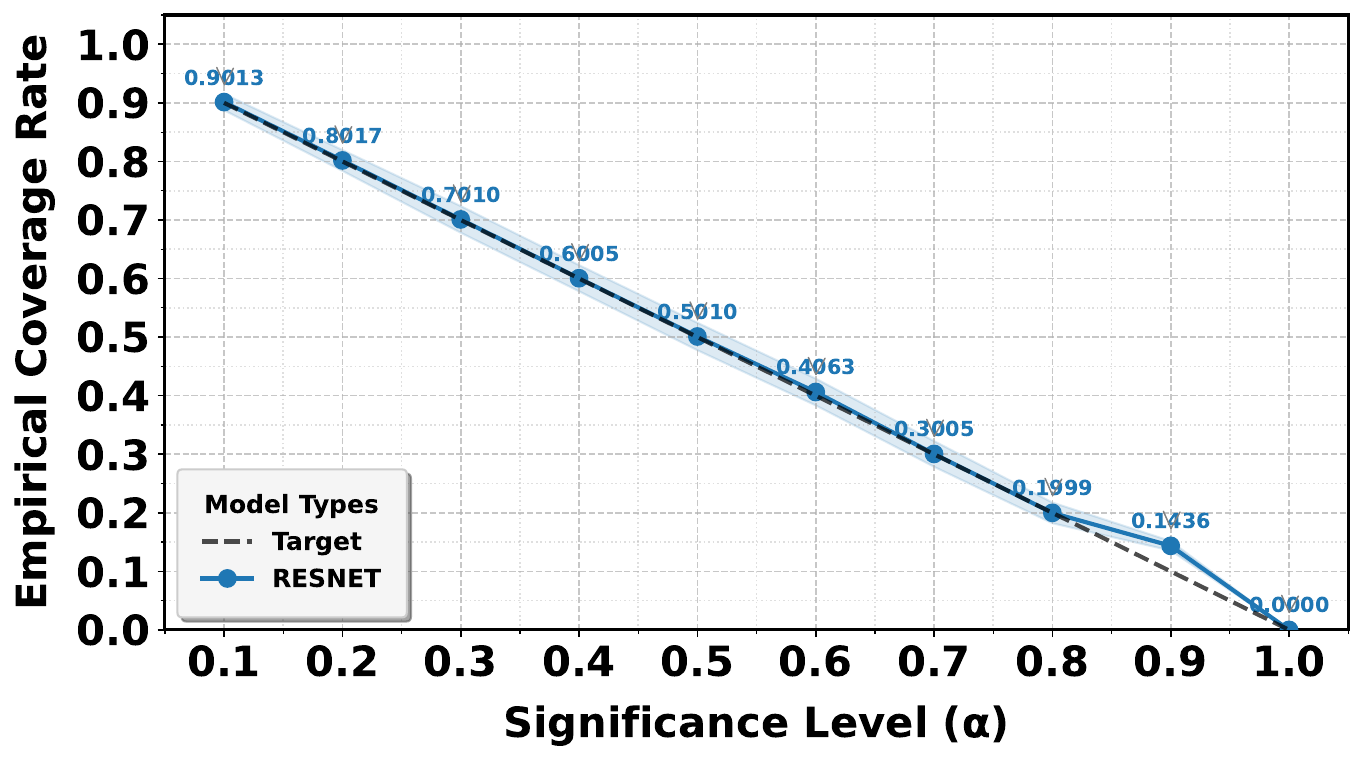}
        \caption{S→C(resnet)}
        \label{fig:sub8}
    \end{subfigure}
    \hfill
    \begin{subfigure}[b]{0.18\textwidth}
        \centering
        \includegraphics[width=\textwidth]{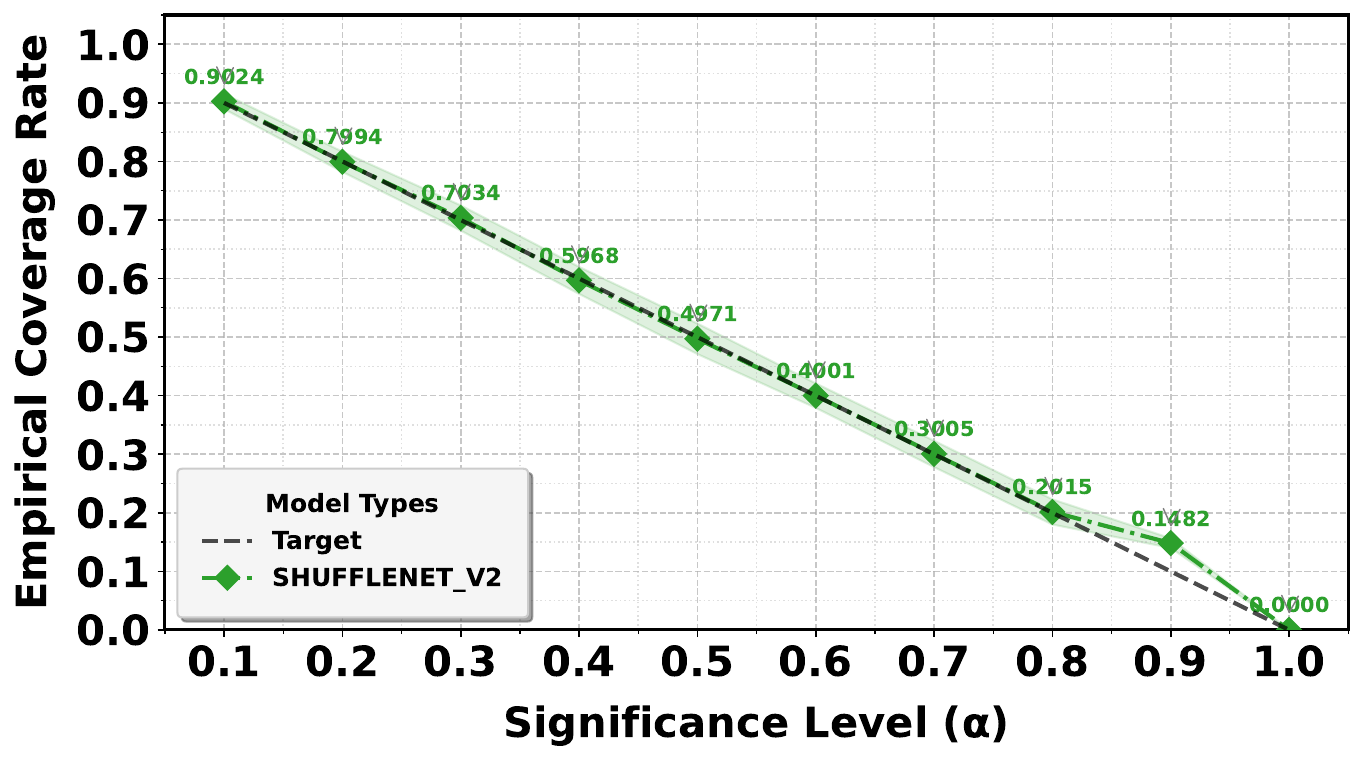}
        \caption{S→C(shufflenetv2)}
        \label{fig:sub9}
    \end{subfigure}
    \hfill
    \begin{subfigure}[b]{0.18\textwidth}
        \centering
        \includegraphics[width=\textwidth]{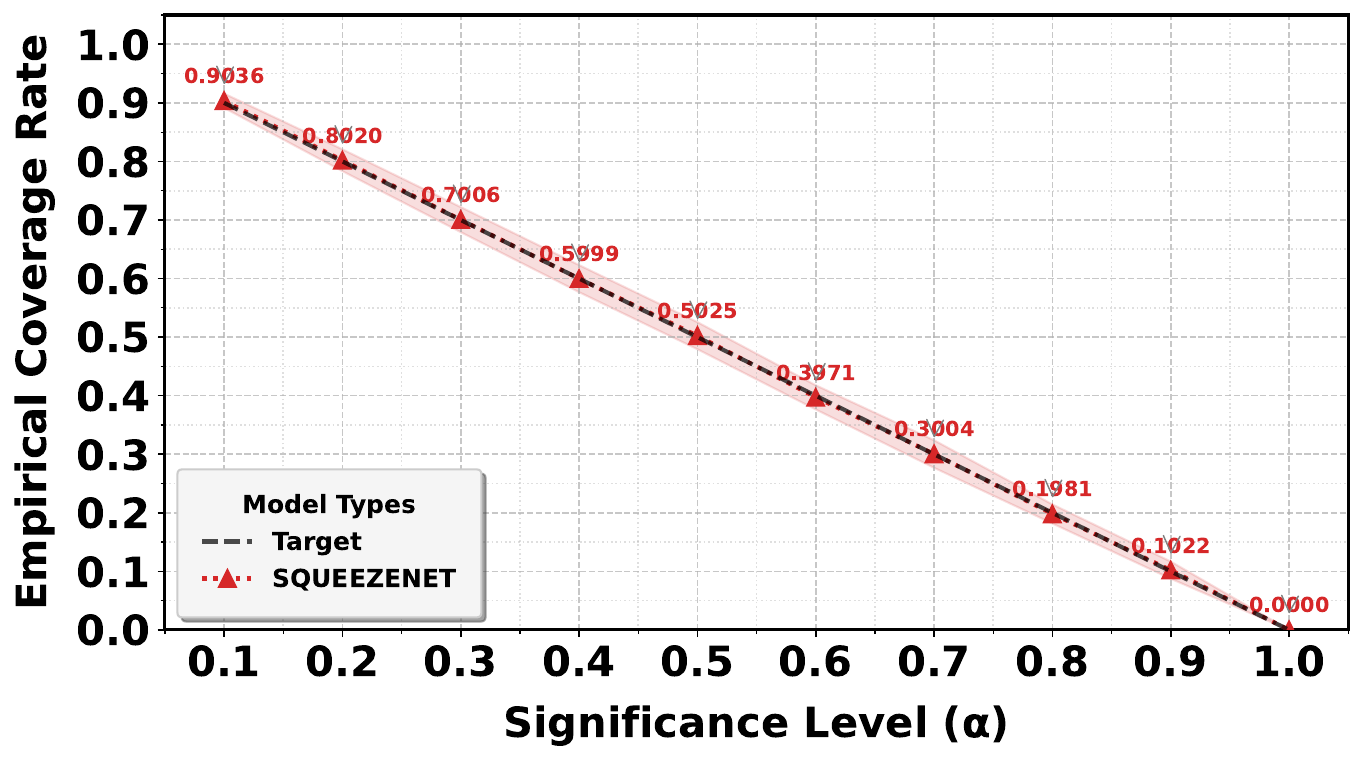}
        \caption{S→C(squeezenet)}
        \label{fig:sub10}
    \end{subfigure}
    \caption{Empirical coverage rate curves of different models in CWRU→CWRU (upper row) and SEU→CWRU (lower row) scenarios (the black dashed line is the target coverage rate, and the shadow indicates the standard deviation of the coverage rate of 100 tests)}
    \label{fig:full1}
\end{figure}

\begin{figure}[!t]
    \centering
    \begin{subfigure}[b]{0.18\textwidth}
        \centering
        \includegraphics[width=\textwidth]{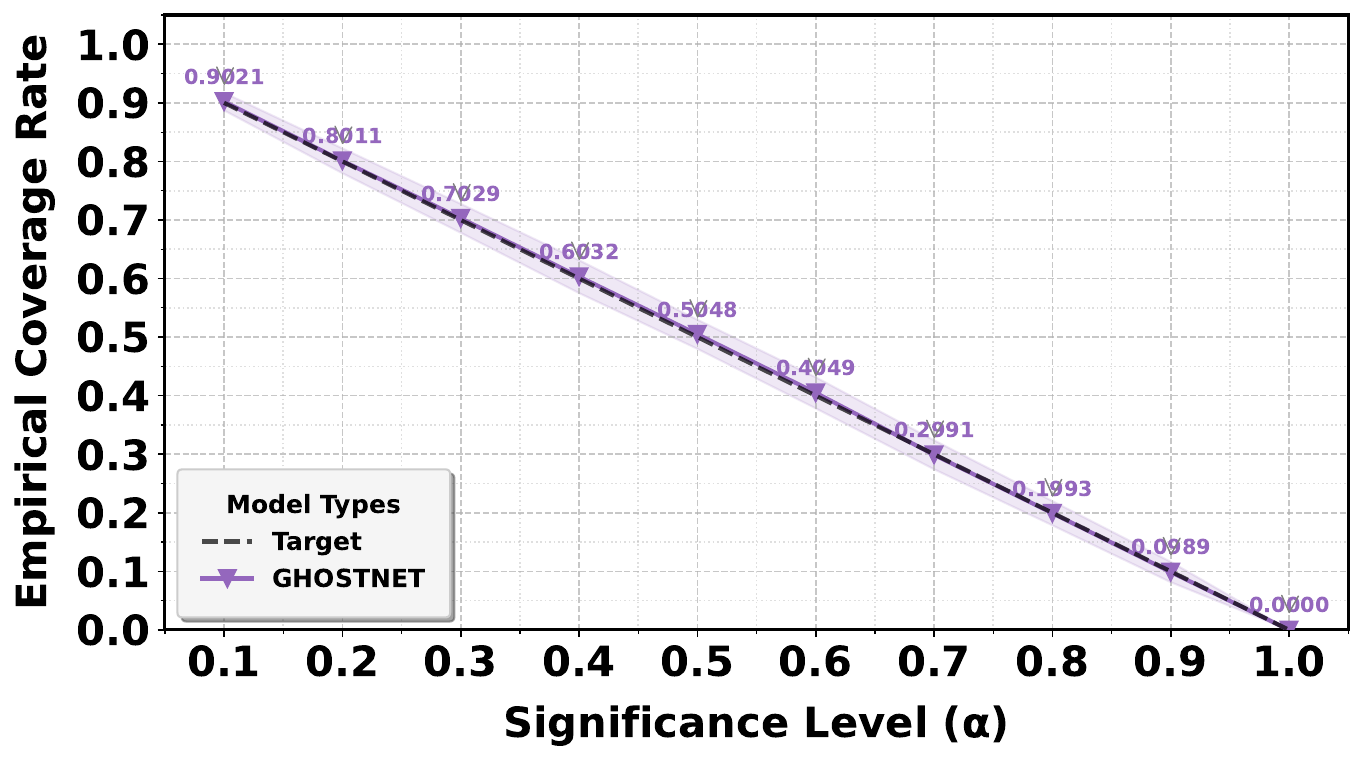}
        \caption{C→S(ghostnet)}
        \label{fig:sub1}
    \end{subfigure}
    \hfill
    \begin{subfigure}[b]{0.18\textwidth}
        \centering
        \includegraphics[width=\textwidth]{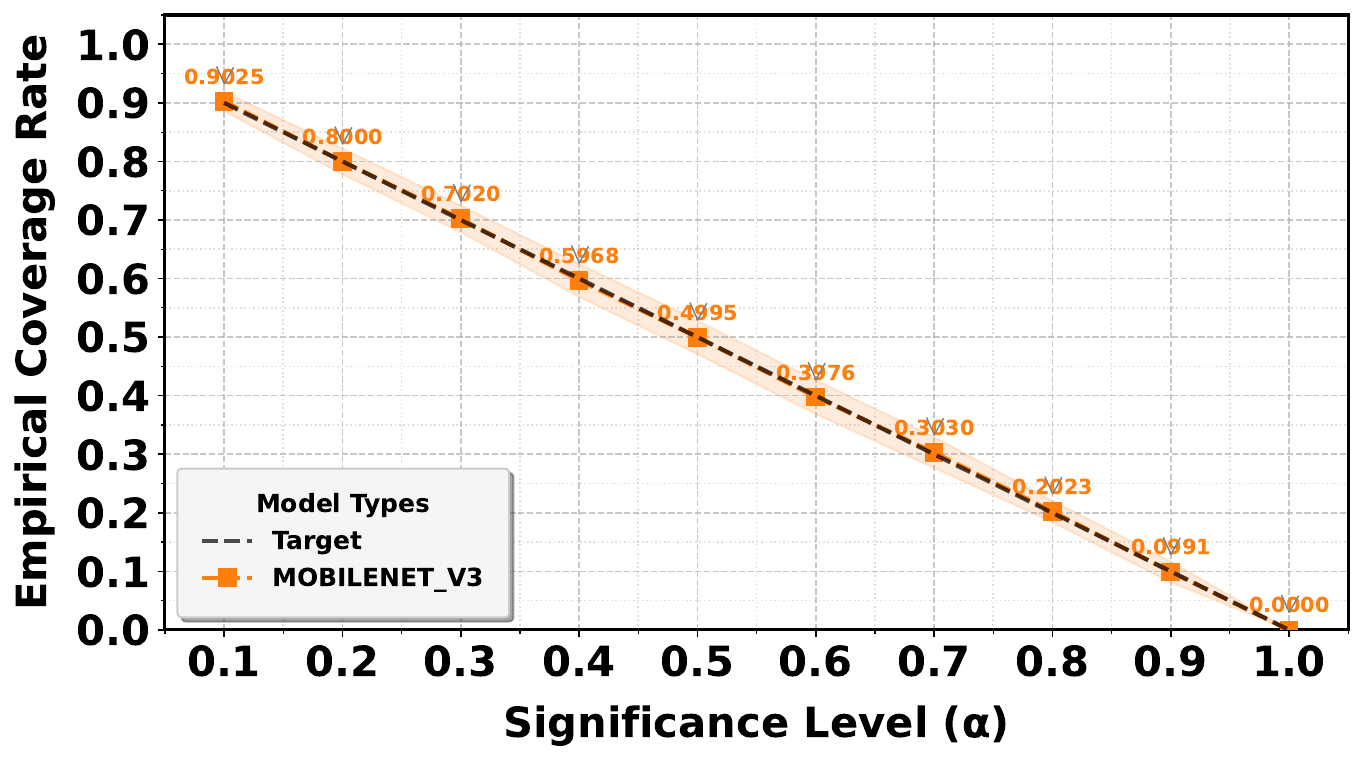}
        \caption{C→S(mobilenetv3)}
        \label{fig:sub2}
    \end{subfigure}
    \hfill
    \begin{subfigure}[b]{0.18\textwidth}
        \centering
        \includegraphics[width=\textwidth]{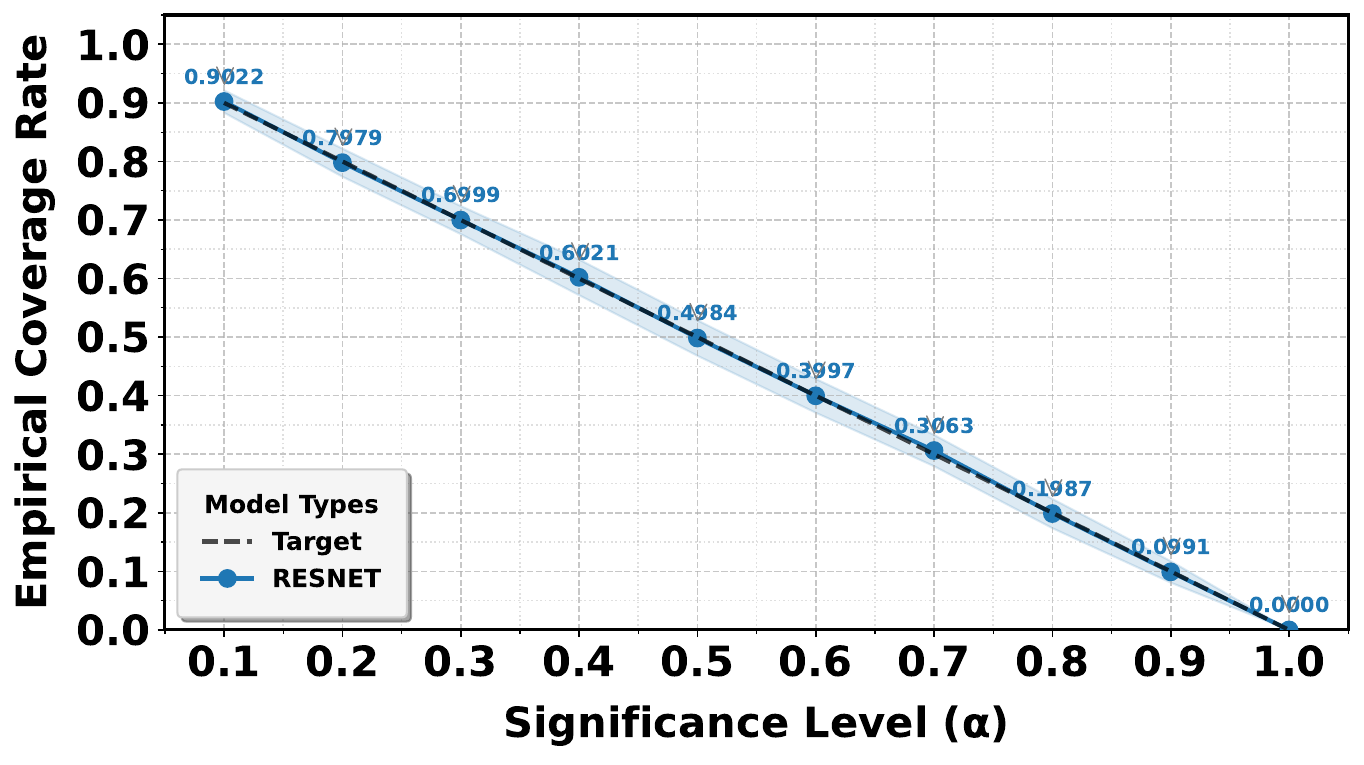}
        \caption{C→S(resnet)}
        \label{fig:sub3}
    \end{subfigure}
    \hfill
    \begin{subfigure}[b]{0.18\textwidth}
        \centering
        \includegraphics[width=\textwidth]{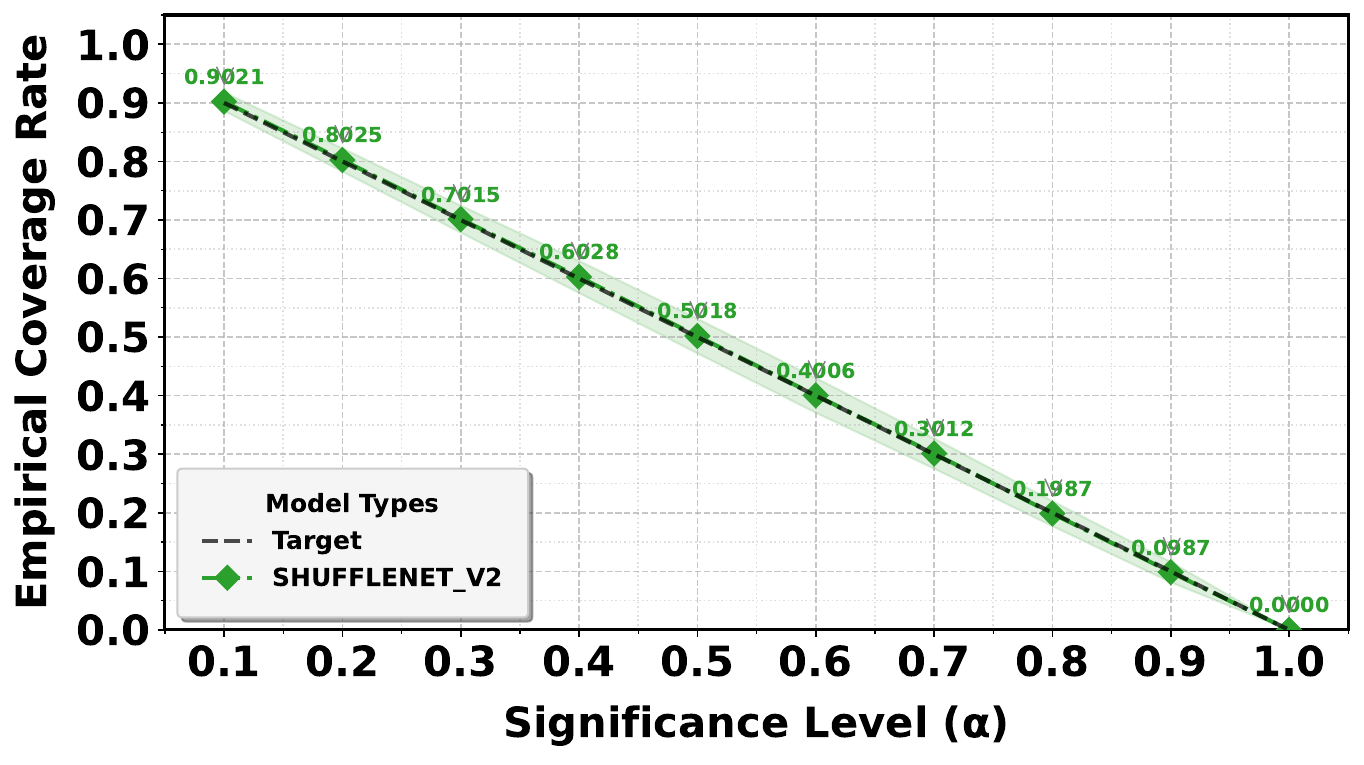}
        \caption{C→S(shufflenetv2)}
        \label{fig:sub4}
    \end{subfigure}
    \hfill
    \begin{subfigure}[b]{0.18\textwidth}
        \centering
        \includegraphics[width=\textwidth]{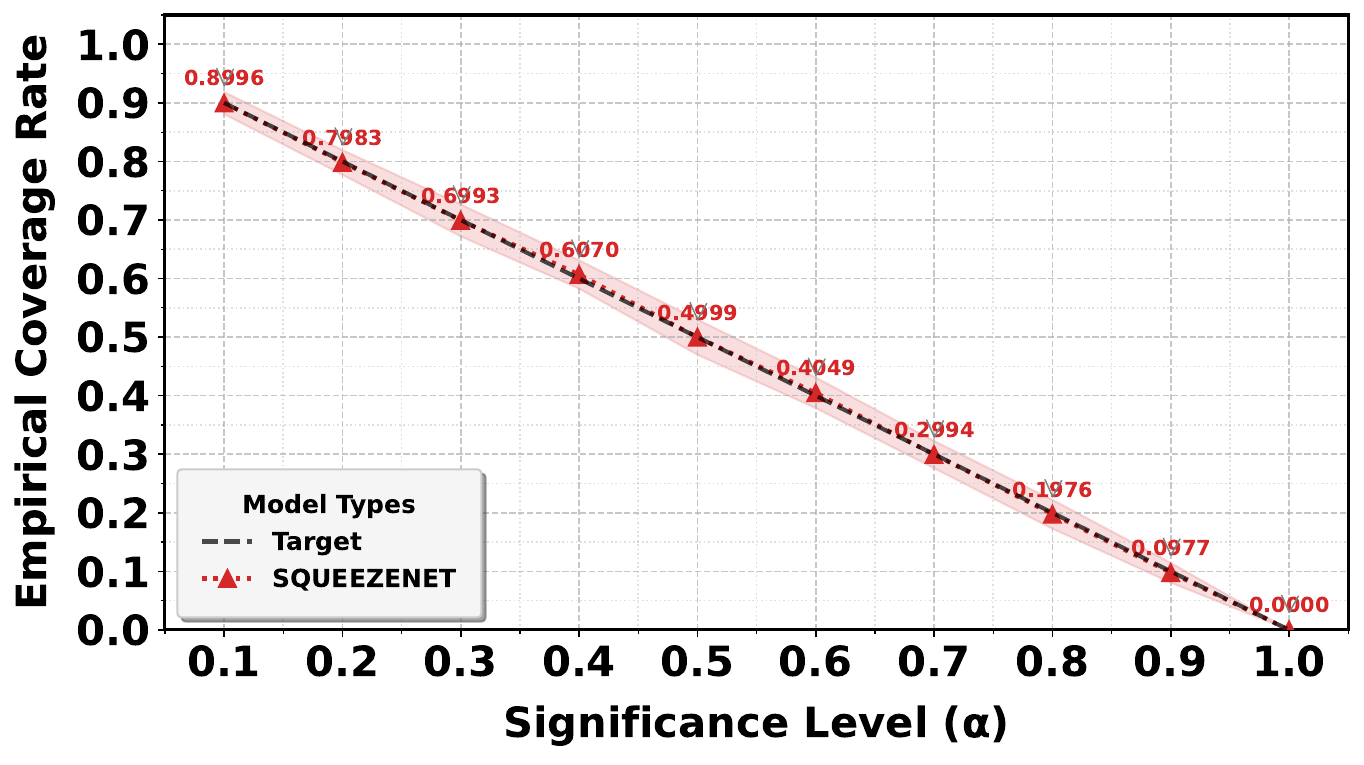}
        \caption{C→S(squeezenet)}
        \label{fig:sub5}
    \end{subfigure}
    
    \vspace{1cm} 
    
    \begin{subfigure}[b]{0.18\textwidth}
        \centering
        \includegraphics[width=\textwidth]{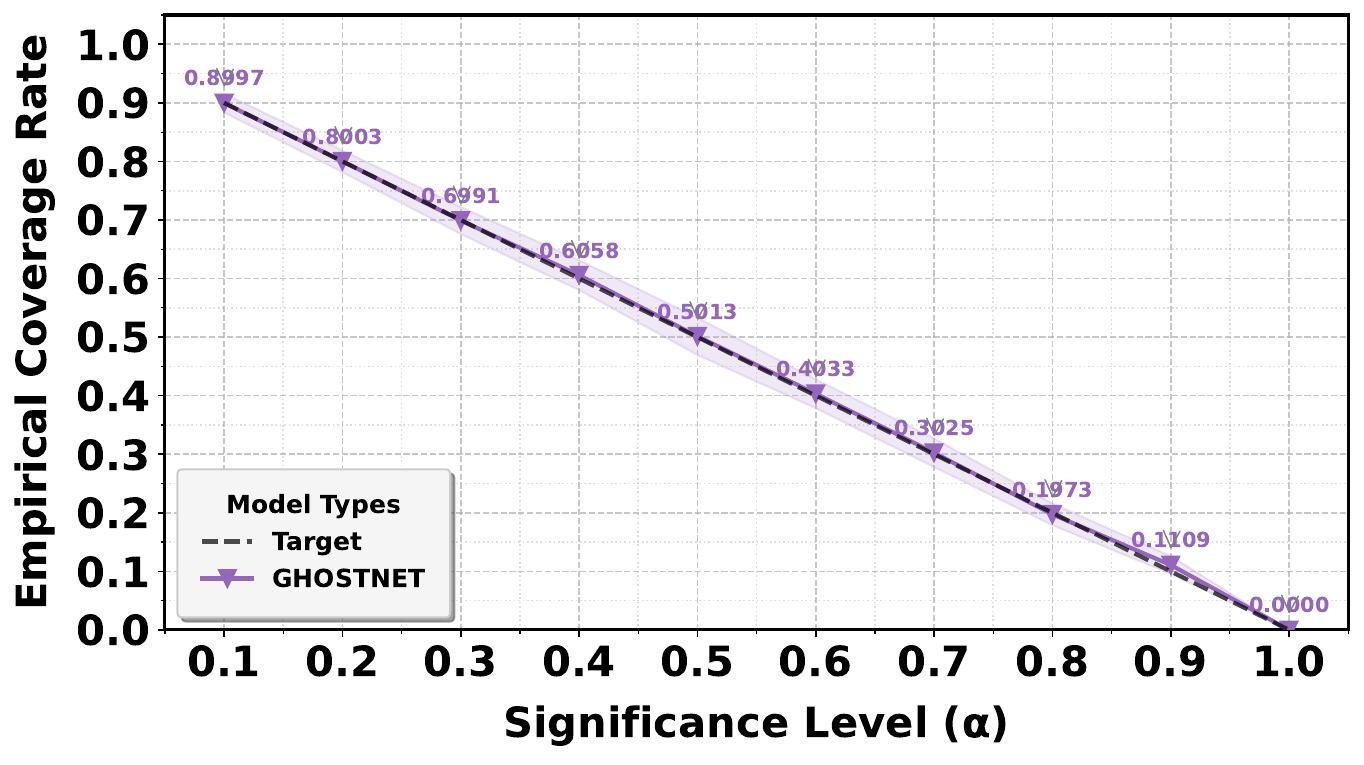}
        \caption{S→S(ghostnet)}
        \label{fig:sub6}
    \end{subfigure}
    \hfill
    \begin{subfigure}[b]{0.18\textwidth}
        \centering
        \includegraphics[width=\textwidth]{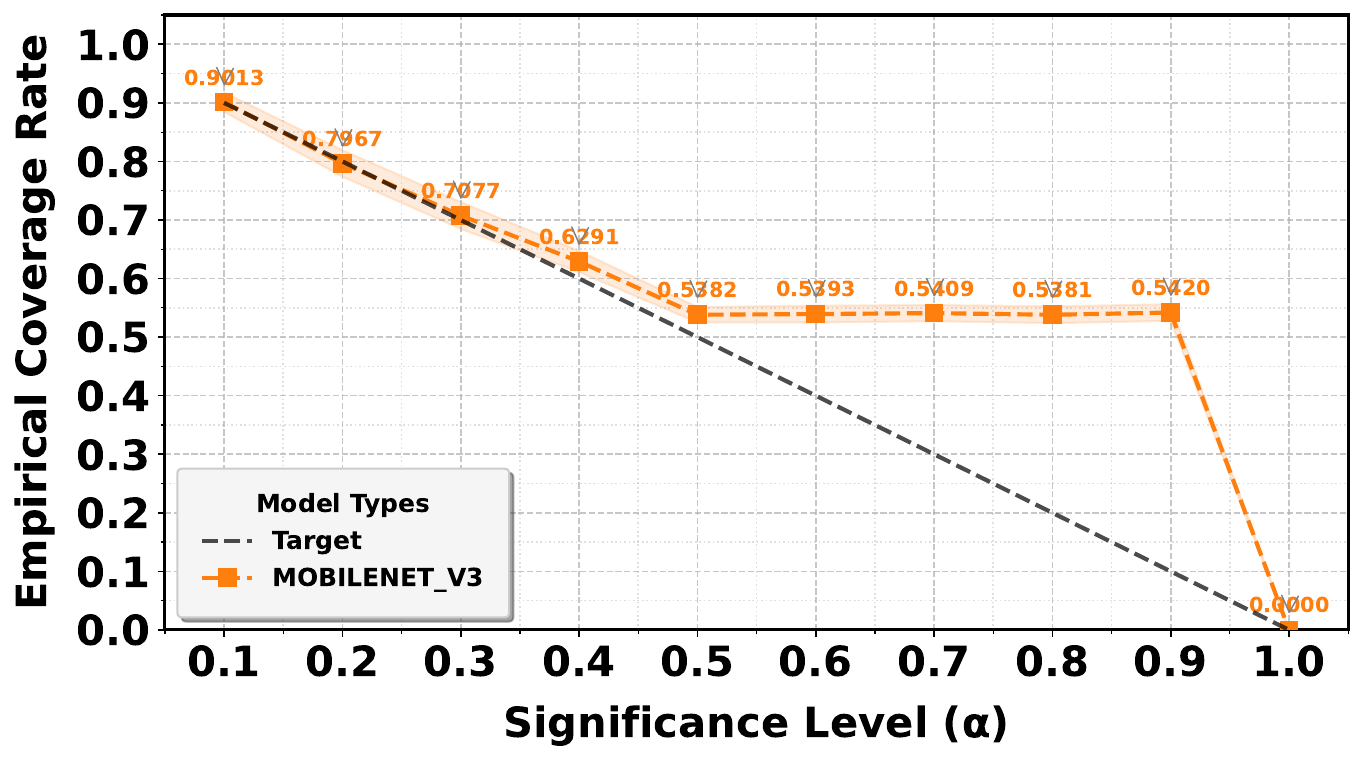}
        \caption{S→S(mobilenetv3)}
        \label{fig:sub7}
    \end{subfigure}
    \hfill
    \begin{subfigure}[b]{0.18\textwidth}
        \centering
        \includegraphics[width=\textwidth]{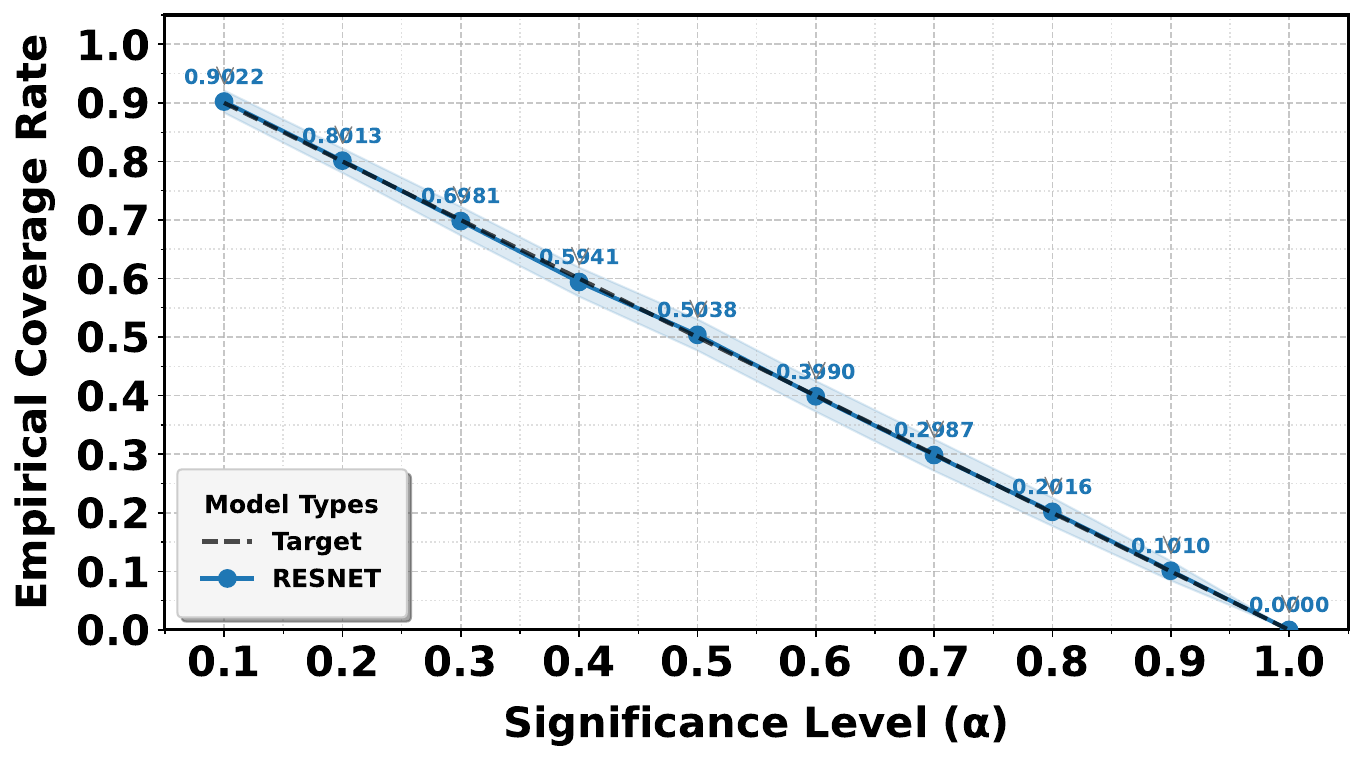}
        \caption{S→S(resnet)}
        \label{fig:sub8}
    \end{subfigure}
    \hfill
    \begin{subfigure}[b]{0.18\textwidth}
        \centering
        \includegraphics[width=\textwidth]{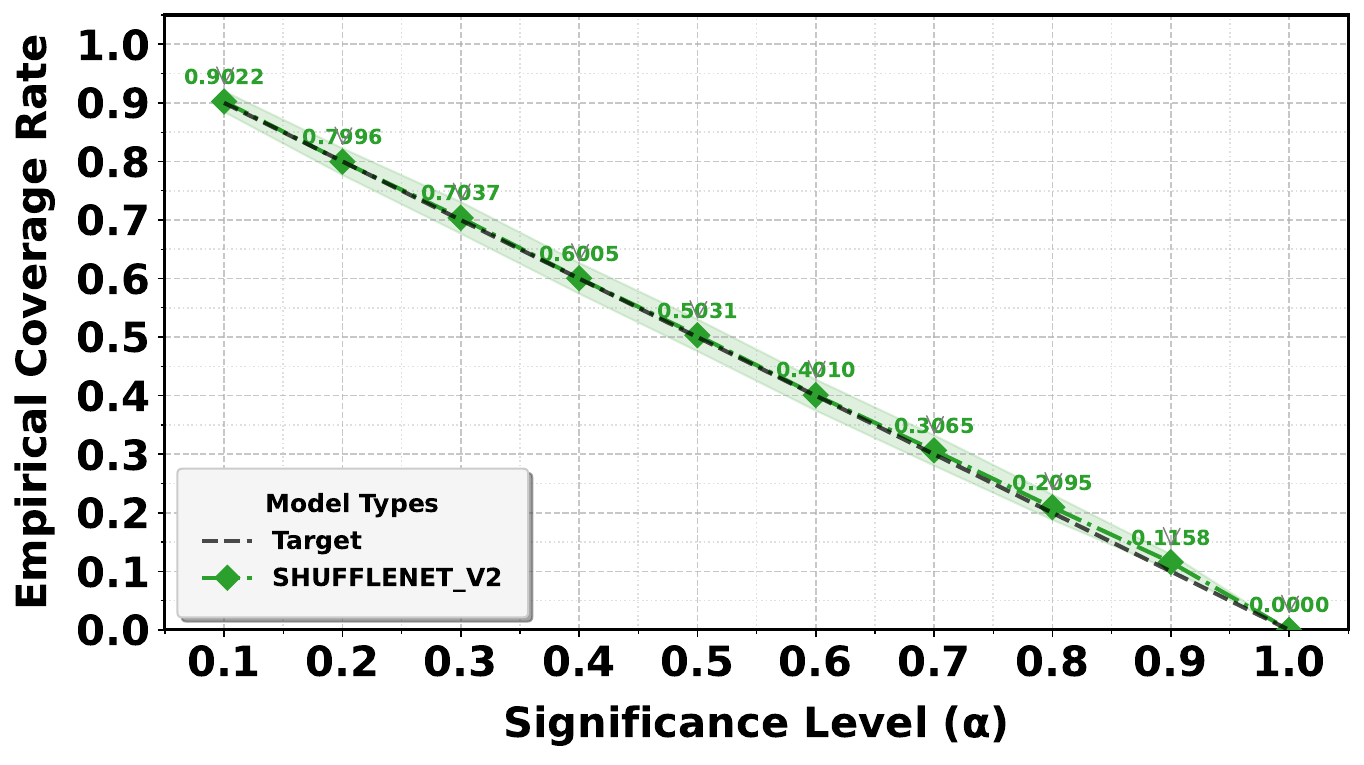}
        \caption{S→S(shufflenetv2)}
        \label{fig:sub9}
    \end{subfigure}
    \hfill
    \begin{subfigure}[b]{0.18\textwidth}
        \centering
        \includegraphics[width=\textwidth]{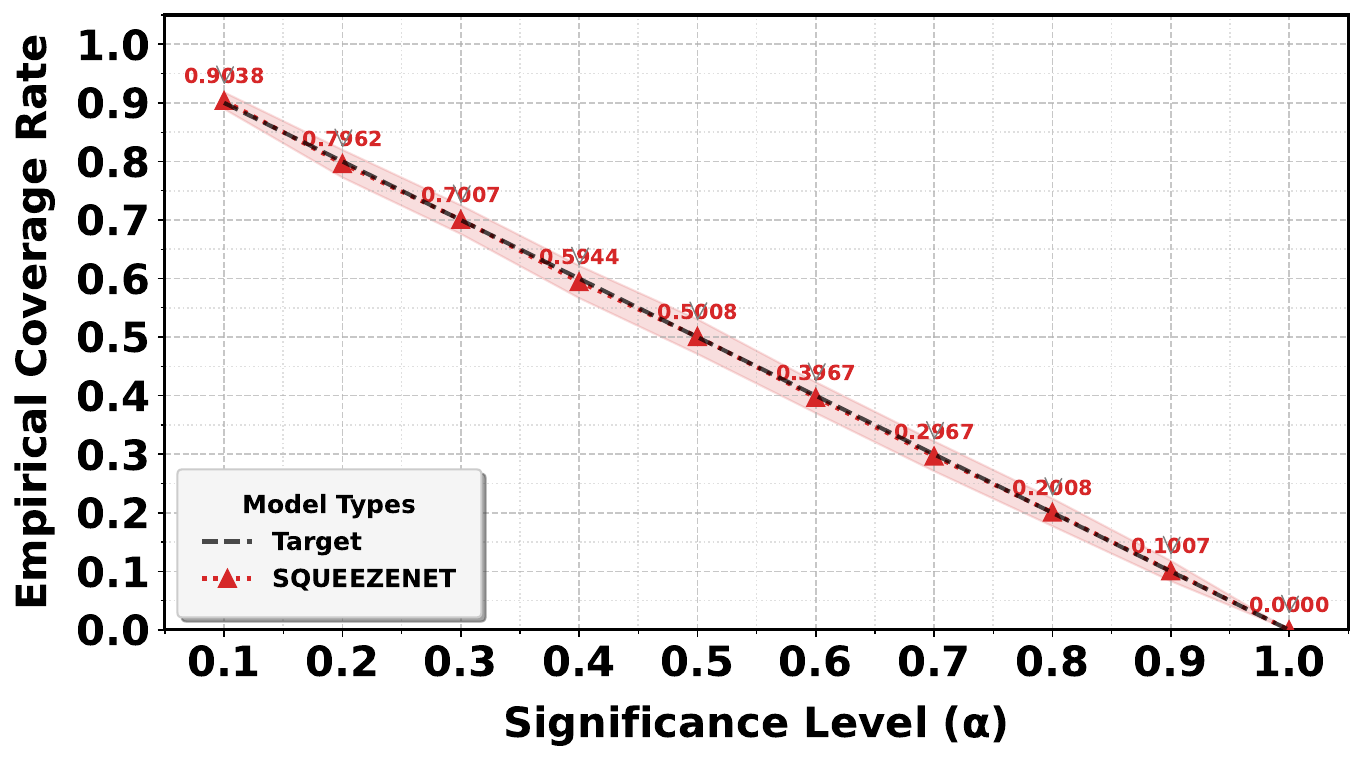}
        \caption{S→S(squeezenet)}
        \label{fig:sub10}
    \end{subfigure}
    \caption{Empirical coverage rate curves of different models in CWRU→SEU (upper row) and SEU→SEU (lower row) scenarios (the black dashed line is the target coverage rate, and the shadow indicates the standard deviation of the coverage rate of 100 tests)}
    \label{fig:full2}
\end{figure}

\begin{table}[!t]
\centering
\footnotesize 
\setlength{\tabcolsep}{3pt}
\caption{The prediction efficiency of models trained on SEU and tested on CWRU, as well as models trained on CWRU and tested on SEU, is quantified by p-value Conformal Prediction, measuring the mean prediction set size (±SD) for each test scenario (with label counts specified in parentheses for each dataset).}
\resizebox{\textwidth}{!}{ 
\begin{tabular}{@{}l l *{9}{c} @{}}
\toprule
\textbf{Dataset} & \textbf{Model} & \multicolumn{9}{c}{\textbf{Risk Level}} \\
\cmidrule(lr){1-2} \cmidrule(lr){3-11}
{} & {} & {0.1} & {0.2} & {0.3} & {0.4} & {0.5} & {0.6} & {0.7} & {0.8} & {0.9} \\
\midrule

\multicolumn{2}{@{}l@{} }{\textbf{Method: p-value Conformal Prediction}} & \multicolumn{9}{c}{} \\ 
\cmidrule{1-11}  

    \multirow{5}{*}{\centering CWRU(4)} 
        & Resnet   & $3.518_{\pm0.039}$ & $3.105_{\pm0.0628}$ & $2.574_{\pm0.091}$ & $2.069_{\pm0.071}$ & $1.728_{\pm0.053}$ & $1.374_{\pm0.049}$ & $0.884_{\pm0.138}$ & $0.467_{\pm0.040}$ & $0.234_{\pm0.010}$ 
        \\
        & Mobilenetv3   & $3.714_{\pm0.282}$ & $2.936_{\pm0.055}$ & $2.614_{\pm0.056}$ & $2.334_{\pm0.053}$ & $2.062_{\pm0.063}$ & $1.670_{\pm0.086}$ & $1.036_{\pm0.049}$ & $0.891_{\pm0.007}$ & $0.893_{\pm0.007}$  \\
        & Shufflenetv2   & $3.320_{\pm0.040}$ & $2.976_{\pm0.055}$ & $2.684_{\pm0.059}$ & $2.304_{\pm0.067}$ & $1.909_{\pm0.076}$ & $1.446_{\pm0.065}$ & $0.959_{\pm0.042}$ & $0.546_{\pm0.060}$ & $0.258_{\pm0.010}$  \\
        & Squeezenet   & $3.324_{\pm0.047}$ & $2.839_{\pm0.055}$ & $2.499_{\pm0.061}$ & $2.135_{\pm0.062}$ & $1.761_{\pm0.068}$ & $1.364_{\pm0.043}$ & $1.027_{\pm0.060}$ & $0.498_{\pm0.062}$ & $0.200_{\pm0.021}$  \\
        & Ghostnet   & $3.730_{\pm0.081}$ & $3.086_{\pm0.056}$ & $2.808_{\pm0.052}$ & $2.504_{\pm0.068}$ & $1.916_{\pm0.125}$ & $1.236_{\pm0.109}$ & $0.889_{\pm0.021}$ & $0.701_{\pm0.022}$ & $0.630_{\pm0.010}$  \\

    \cmidrule{2-11}  
    \multirow{5}{*}{\centering SEU(4)}  
     & Resnet   & $3.008_{\pm0.079}$ & $2.313_{\pm0.058}$ & $1.908_{\pm0.050}$ &
     $1.583_{\pm0.056}$ & $1.367_{\pm0.040}$ &
     $1.160_{\pm0.030}$ & $0.974_{\pm0.031}$ &  $0.740_{\pm0.044}$ & $0.145_{\pm0.044}$ \\
    & Mobilenetv3   & $3.550_{\pm0.035}$ & $3.185_{\pm0.049}$ & $2.141_{\pm0.066}$ & $1.751_{\pm0.043}$ & $1.485_{\pm0.046}$ & $1.198_{\pm0.037}$ & $1.059_{\pm0.019}$ & $0.556_{\pm0.055}$ & $0.215_{\pm0.030}$  \\
    & Shufflenetv2   & $3.373_{\pm0.041}$ & $3.039_{\pm0.035}$ & $2.480_{\pm0.065}$ & $2.103_{\pm0.060}$ & $1.743_{\pm0.078}$ & $1.314_{\pm0.071}$ & $1.081_{\pm0.025}$ & $0.793_{\pm0.039}$ & $0.514_{\pm0.040}$  \\
    & Squeezenet   & $3.843_{\pm0.035}$ & $3.483_{\pm0.074}$ & $2.986_{\pm0.091}$ & $2.537_{\pm0.081}$ & $1.886_{\pm0.111}$ & $1.402_{\pm0.040}$ & $1.165_{\pm0.038}$ & $0.929_{\pm0.022}$ & $0.761_{\pm0.025}$  \\
    & Ghostnet   & $3.228_{\pm0.053}$ & $2.726_{\pm0.075}$ & $2.326_{\pm0.068}$ & $1.954_{\pm0.062}$ & $1.518_{\pm0.070}$ & $1.232_{\pm0.032}$ & $0.909_{\pm0.065}$ & $0.439_{\pm0.031}$ & $0.241_{\pm0.026}$  \\
\bottomrule
\end{tabular}
}

\label{tab:framework_comparison}
\end{table}
\subsection{Feature Extraction and Time-Frequency Imaging}
For each raw vibration signal sample from the dataset, we applied the CWT using the Complex Morlet wavelet as the mother wavelet. The analysis was configured with a sampling frequency of 12~kHz, corresponding to the data acquisition rate, and a total of 128 scales to capture a wide range of frequency components. The resulting wavelet coefficients, which represent the signal's energy distribution across time and frequency, were converted to their absolute values. These magnitude matrices were then visualized as scalograms and saved as 224$\times$224 pixel color images in PNG format, a dimension compatible with standard pre-trained CNN architectures.  This procedure converted the entire 1D signal dataset into a corresponding 2D time-frequency image dataset, making it suitable for direct input into various CNN models for fault classification.

\subsection{Model Training}
We trained and evaluated five different Convolutional Neural Network (CNN) architectures: ResNet, MobileNetV3, ShuffleNetV2, SqueezeNet, and GhostNet. The final classification layer of each model is suitable for outputting predictions of four defined fault categories (normal, inner circle fault, ball fault, outer circle fault). All models have been trained to converge on the original dataset.

\subsection{Cross dataset evaluation}
A critical challenge for deep learning-based fault diagnosis models is their ability to generalize to data from different machines or operating conditions, a problem known as domain shift. To quantify this challenge, we conducted a cross-domain accuracy evaluation using the trained models and a second public dataset, the SEU dataset, denoted as `S`. The original CWRU dataset is denoted as `C`. The experiment involved two scenarios: training on `C` and testing on `S` (C$\to$S), and training on `S` and testing on `C` (S$\to$C).

The results, presented in Table~\ref{tab:cross_ser_results}, show a severe degradation in performance for all models when faced with data from an unseen domain. The accuracies plummeted to a range of 24\% to 36\%, which is only slightly better than random guessing for a four-class problem (25\%). This outcome starkly illustrates the poor generalization capability of the models, which tend to overfit to the specific statistical distribution of the source domain data. This limitation underscores the unreliability of deterministic point predictions in real-world scenarios where data distributions can vary. Consequently, this finding motivates our subsequent adoption of Calibrated Fault Detection with p-value Conformal Prediction, a method designed to produce statistically rigorous prediction sets with guaranteed coverage, thereby providing a more reliable measure of uncertainty.

\subsection{Calibrated Fault Detection with p-value Conformal Prediction}

To address the unreliability of point predictions in cross-domain scenarios, we employed calibrated fault detection with p-value conformal prediction to generate statistically rigorous prediction sets. This method provides a formal measure of confidence and guarantees a user-specified coverage rate. The procedure was implemented by first partitioning the target domain's test data into a proper calibration set and a validation set. For each sample in the calibration set, we computed conformity scores, which correspond to the softmax probability of the true class.

For a new sample from the validation set, a p-value is calculated for each possible fault class. This p-value quantifies how "unusual" the sample appears under a hypothetical label, relative to the samples in the calibration set. The final prediction set, $C(x)$, is then constructed by including all classes whose p-value exceeds a predefined significance level, $\alpha \in [0, 1]$. This framework mathematically guarantees that the long-run frequency of including the true label in the prediction set is at least $1-\alpha$, i.e., $P(\text{true label} \in C(x)) \geq 1-\alpha$. To evaluate this, we use two metrics:ECR to verify the coverage guarantee, and the APSS to assess the method's efficiency and uncertainty-quantification capability.

The experimental results validate the theoretical properties of our approach. Figure~\cref{fig:full1,fig:full2} demonstrates that for all five models and across both cross-domain scenarios (C$\to$S and S$\to$C), the empirical coverage rate consistently meets or exceeds the target coverage rate of $1-\alpha$. This holds true even for the models that exhibited poor point-prediction accuracy, confirming the robustness of the coverage guarantee.

Furthermore, Table~\ref{tab:framework_comparison} shows the relationship between the risk level $\alpha$ and the average prediction set size. As the risk level $\alpha$ increases (i.e., the desired coverage $1-\alpha$ decreases), the average size of the prediction sets becomes smaller. This illustrates that the prediction set size serves as an intuitive and direct indicator of model uncertainty: We set risk level small, in which case a small prediction set reflects high confidence in the model's predictions, while a large prediction set implies high uncertainty because the model cannot clearly distinguish between multiple potential fault types. This calibrated uncertainty measure is essential for making reliable decisions in safety-critical applications.

\section{Conclusion}
In this study, we introduced a novel fault detection framework by integrating conformal prediction with significance testing. This approach successfully addresses the critical challenge of unreliable performance and the lack of risk guarantees in traditional models, particularly under varying data distributions.Our key contribution is a method that provides a strict, theoretically-grounded guarantee: the false alarm rate is rigorously controlled under a user-specified significance level $\alpha$. Experiments on bearing fault diagnosis datasets confirm this guarantee, demonstrating the framework's reliability. Furthermore, the size of the generated prediction set serves as a direct measure of model uncertainty, enhancing the interpretability of the results.
\bibliographystyle{unsrt}  
\bibliography{references}

\end{document}